\documentclass{svproc}

\usepackage{url}

\usepackage{amssymb}
\usepackage{amsmath}
\usepackage{algorithm}
\usepackage[noend]{algpseudocode}
\usepackage{subfig}
\usepackage{layouts} 
\usepackage{siunitx}
\usepackage{booktabs}
\usepackage{multirow}

\newif\ifusePlainBib
\usePlainBibtrue

\ifusePlainBib
\else
\usepackage[backend=bibtex,maxnames=5,sorting=none,style=ieee,firstinits=true]{biblatex}
\AtBeginBibliography{\small}
\fi

\newcommand{\vamp}{{\sc vamp}}
\newcommand{\powerset}[1]{\mathbb{P}(#1)}
\usepackage{todonotes}

\newif\ifcanbeverbose
\canbeverbosefalse

\newcommand{\proofInSupp}{The proof can be found in the Supplementary Material.}
\newif\ifForWAFR
\ForWAFRfalse

\begin{document}
\mainmatter              
\title{Look before you sweep: \\Visibility-aware motion planning}
\titlerunning{\vamp{}}  

\author{Gustavo Goretkin \and Leslie Pack Kaelbling \and Tom\'{a}s Lozano-P\'{e}rez}

\authorrunning{Gustavo Goretkin et al.} 
\institute{
	Massachusetts Institute of Technology, Cambridge MA 02139, USA,\\
	\email{\{goretkin, lpk, tlp\}@csail.mit.edu},\\
}

\maketitle              

\begin{abstract}
	This paper addresses the problem of planning for a robot with a directional obstacle-detection sensor that must move through a cluttered environment.
	The planning objective is to remain safe by finding a path for the complete robot, including sensor, that guarantees that the robot will not move into any part of the workspace before it has been seen by the sensor.
	Although a great deal of work has addressed a version of this problem in which the ``field of view'' of the sensor is a sphere around the robot, there is very little work addressing robots with a narrow or occluded field of view.
	We give a formal definition of the problem, several solution methods with different computational trade-offs, and experimental results in illustrative domains.
\end{abstract}

\section{Introduction}

Consider a mobile-manipulation robot that must move through a crowded environment.
If the location of all obstacles in the environment is known, then the problem it faces is a familiar motion-planning problem.
But if the environment can contain unknown obstacles, then the robot must incorporate sensing into its plan in order to guarantee that it will not collide with anything.
In one extreme version of this problem, the environment is entirely unknown, and would best be treated with a combination of map-building and exploration.
We will focus on a different regime, that arises in the case of a household robot:  the primary obstacles in the domain (e.g. walls, refrigerators) are known but there are other temporary obstacles (e.g., toys, trash cans, lightweight chairs).
In this case, it is worthwhile to plan a path to a target configuration, but the path must take visibility into account, making sure that it never moves into a region of space that it has not already observed.
Should the robot encounter an unexpected object, it could then make a new plan taking it into account, or move it out of the way.

When we speak of visibility, we mean any robot-mounted ability to gather information about the locations of obstacles in its neighborhood.
It could be based on vision, lidar, or even the ability to reach out slowly with a hand and sense contact or lack thereof.
If a robot has visibility of a sphere around it in workspace, and it is quasi-static, then the problem of safe movement is simple: the robot must just move in small enough steps that it never exits the sphere it saw from its previous configuration.
A great deal of work has addressed exploration problems in this visibility model.
However, many robots have less encompassing sensing.
Figure~\ref{sensors} illustrates several different sensor configurations for a simple planar robot.
Case (a) reflects the most common assumption about sensing: that is, that the robot can perceive a ball of some radius in the workspace (although this is often described as perceiving a ball in configuration space, which is not necessarily sensible); case (b) shows a wide field of view as might occur with some steerable sensors; case (c) shows a narrow view as might occur in some vision sensors; case (d) occurs for many humanoid robots with a camera mounted on the head: although they can see a view cone in front of them, it is occluded by the body and so there is a region of space immediately in front of the robot that cannot be seen; and case (e) illustrates a situation in which, for example, a humanoid is carrying a large box in front of it, so its field of view is split into two narrow cones.
Our approach handles a general mapping from robot configurations to ``viewed'' regions of workspace, encompassing these examples, and even more complex ones, including cameras mounted on the arms of a mobile manipulator.

\begin{figure}
	\begin{center}
		\includegraphics[width=1.0\textwidth]{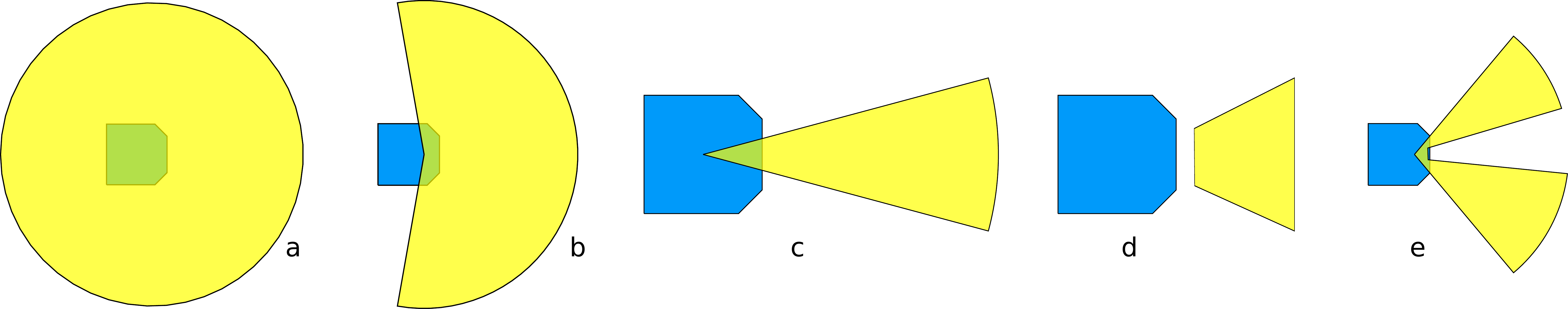}
	\end{center}
	\caption{Some possible sensed volumes with respect to robot configuration.}
	\label{sensors}






\end{figure}

\begin{figure}
	\begin{center}
		\subfloat 
		{
			\includegraphics[width=0.25\textwidth]{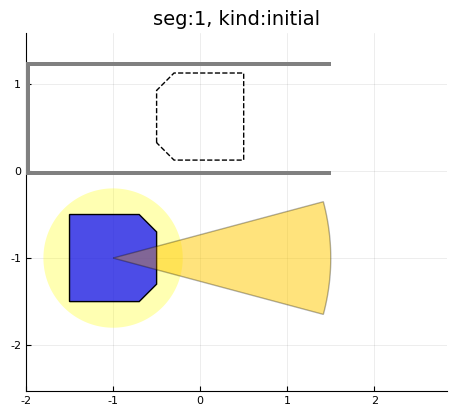}
		}
		\subfloat
		{
			\includegraphics[width=0.25\textwidth]{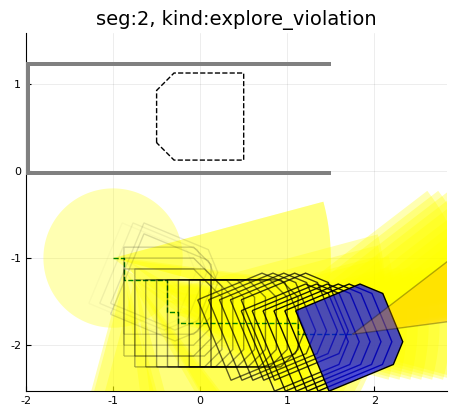}
		}
		\subfloat
		{
			\includegraphics[width=0.25\textwidth]{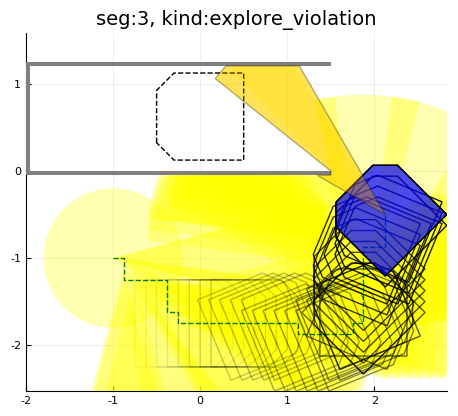}
		}
		\subfloat
		{
			\includegraphics[width=0.25\textwidth]{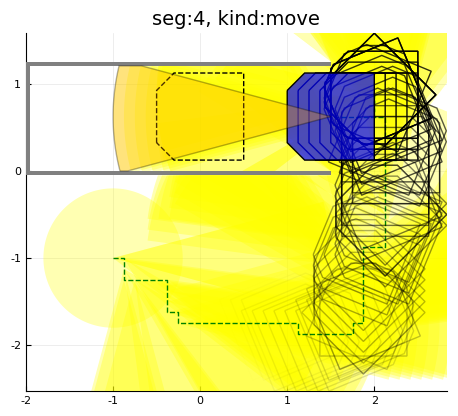}
		}
	\end{center}
	\caption{In the {\sc HallwayEasy} domain, a robot with narrow (\ang{30}) field of view must steer carefully around a corner.}
	\label{narrowCorner}
\end{figure}

\begin{figure}
	\begin{center}
		\subfloat
		{
			\includegraphics[width=0.25\textwidth]{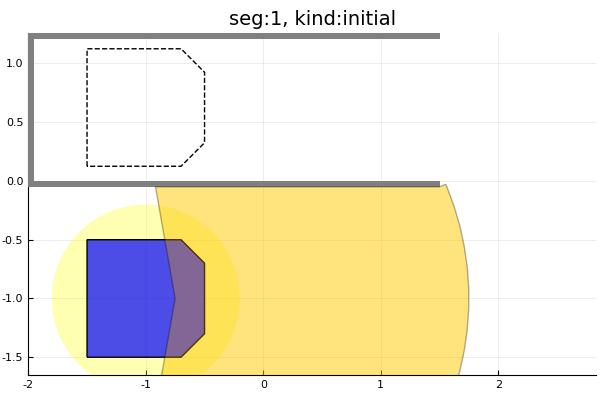}
		}
		\subfloat
		{
			\includegraphics[width=0.25\textwidth]{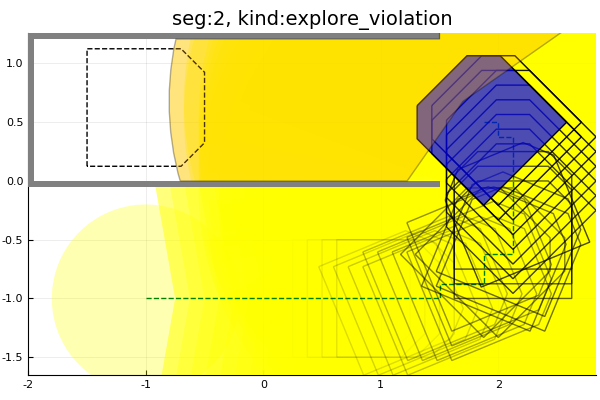}
		}
		\subfloat
		{
			\includegraphics[width=0.25\textwidth]{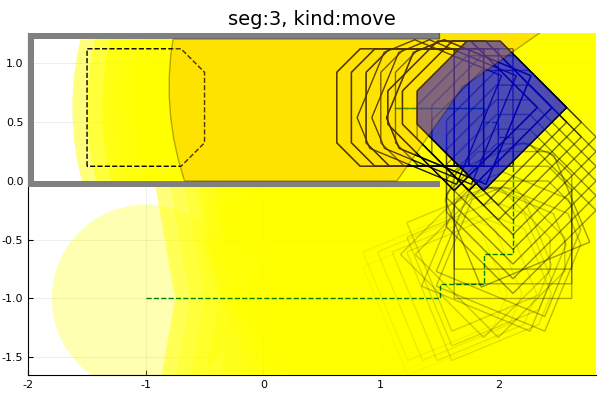}
		}
		\subfloat
		{
			\includegraphics[width=0.25\textwidth]{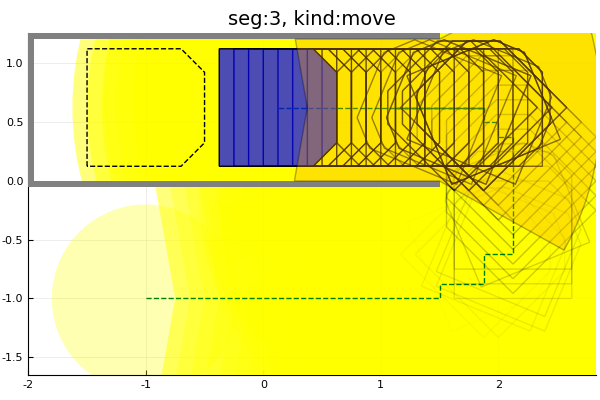}
		}
	\end{center}
	\caption{In the {\sc HallwayHard} domain, a robot with a wide field of view, must enter the hallway camera-first, back out, re-orient, and back in.}
	\label{backup}

\end{figure}

For a large robot with a limited view that is navigating in a cluttered environment, the problems of moving to the goal and of observing to guarantee safety are generally inextricably linked.
Figure~\ref{narrowCorner} shows a plan for a robot with a narrow view cone to enter a hallway.
The goal is depicted in the dashed outline.
The yellow shading indicates the region of workspace that has been seen by the plan.
Note that the robot has to rotate back and forth as it moves, and swing wide around the corner, in order to guarantee safety;  this path was obtained with the {\sc Vamp\_Backchain} method, described in section~\ref{sec:backchain}.
Figure~\ref{backup} shows a robot with a wider field of view that must enter a narrow hallway backwards.
Because it cannot turn around inside the hallway, it must first look down the hallway then back out and turn around.
Finally, figure~\ref{subfig:two-hallway-subgoals} shows a particularly difficult environment, in which the robot must enter one hallway to look through a narrow opening to gain visibility of a different hallway before entering it.

The contributions of this paper are twofold.
First, we clearly frame the problem of planning a path that is safe, in the sense that it never moves through previously unobserved space, for a general visibility function that maps robot configurations to regions of workspace that are observed.
We call this problem class \vamp{}, for {\it visibility-aware motion planning}.
Second, we supply several algorithms that are complete for this problem, which occupy different points in the trade-off space between path length and planning time.
In particular, we present in section~\ref{ssec:vampalg} an algorithm that is well-suited to solving \vamp{} problems.
The other algorithms we present either provide conceptual clarity to the belief-space nature of the \vamp{} problem, or serve as useful subprocedures for the algorithm in section~\ref{ssec:vampalg}.

The problem is quite computationally complex, because the state of the planning problem includes what space has already been observed, and so the solutions include paths that a traditional motion-planner would never generate, because the solution paths revisit the same robot configuration with different visibility states.
Although the examples presented here are for a robot with 3D configuration space in a 2D workspace, there is nothing in the formulation or algorithms that restricts their applicability to this case.
Our algorithms are provably always correct (will not generate an illegal plan); they are complete for holonomic robots, as well as for non-holonomic robots that can always reverse a path they have traveled, without increasing the total swept volume.

\section{Related work}

There is a body of work, related to ours, that addresses the problem of robot motion planning in the presence of uncertainty about the environment.
This work varies considerably in its objectives and assumptions about the problem.
We will lay out the space of approaches and situate our work within it.


\paragraph{Problem variations:}
There are many possible objectives for the robot's motion.
{\it Coverage} problems generally assume a known map of obstacles but require planning a path that will observe all parts of the reachable space, for example, to inspect a structure~\cite{galceran2013survey,davis2016c-opt,englot2012sampling}.
{\it Exploration} problems generally assume no prior knowledge of the obstacles and desire that the robot eventually observe all parts of the space, possibly while building a map of the obstacles~\cite{yamauchi1997frontier,heng2015efficient,oriolo2004srt,dornhege2013frontier,bircher2016receding,bekris2007greedy,lauri2016planning}.
{\it Navigation} problems seek to reach a specified goal region in an incompletely known environment~\cite{Janson-RSS-18,arvanitakis2018synergistic,richter2018bayesian}.

An additional source of variation is the notion of {\it safety}:  one might wish to guarantee safety (non-collision with any obstacle) absolutely, e.g.~\cite{englot2012sampling,bircher2016receding}, or with high probability with respect to a distribution over obstacles, e.g.~\cite{axelrod2018provably,richter2018bayesian} and with respect to obstacles that may move, e.g.~\cite{bouraine2012provably}.

Formulations vary in their assumptions about observations.
Sometimes no observations are assumed during path execution, as in coverage problems~\cite{galceran2013survey}.
In other cases, an observation is made from some or all states of the robot during execution; the observation depends on the underlying true world map as well as on the robot's state and may be an arbitrary function of those inputs.
Typically, observations are assumed to be in the plane and take the form of either a fixed cone or a circle centered on the robot's location, although more general 3D observations have been considered~\cite{dornhege2013frontier}.
In addition, the robot is typically assumed to be small relative to the obstacles and the environment uncrowded, so that the robot can be approximated as a point.
These simplifying assumptions blur the distinction between workspace and configuration space and limit the application of these algorithms to more general settings, such as those arising in mobile manipulation.

Robot motion is typically assumed to be planar.
However, during mobile manipulation, all the degrees of freedom of the robot may affect observations, e.g. the arms may partially block the sensor; this setting motivates our work although our experiments are in lower dimensions.
Previous work has considered a variety of motion models: kinematic, whether holonomic or non-holonomic~\cite{bekris2007greedy}, or kino-dynamic~\cite{Janson-RSS-18,bry2011rapidly}, and with deterministic or noisy actuation~\cite{lauri2016planning,arvanitakis2018synergistic}.
Robot dynamics introduce additional difficulty because the robot is not able to stop instantly when an obstacle is detected; instead, it must ensure that it never enters an {\it inevitable collision state}({\sc ics})~\cite{Janson-RSS-18,bekris2007greedy} with respect to its current known free space and a motion model of possible obstacles.

It is worth highlighting a harder case of the coverage problem, known as visibility-based pursuit-evasion~\cite{pursuit_evasion_limited_fov2006, pursuit_evasion_optimal2017}.
A solution to this problem is a path, which the pursuer follows, that guarantees that any evaders will be detected.
Solution paths to this problem necessarily cover the whole space,
In general the paths must revisit the same robot configuration with different information states.


\paragraph{Solution strategies:}
When there is no map uncertainty, as in coverage problems, then the solution is a (minimum-length) path that covers all parts of the space with its sensors.


Problems with map uncertainty, as in our case, can all be cast as some version of a partially observed Markov decision process ({\sc pomdp}).
For our version of the problem, the {\it state space} would be the cross product of the robot's configuration space with the space of all possible arrangements of free/occupied space defined by the unknown obstacles, the actions would be finite linear motions in configuration space, and the observations would be determined by the visibility function of the sensor.
The objective would be to minimize path length, but with the hard constraint of not moving into non-viewed areas (making this not a standard {\sc pomdp}).
Seeing the problem this way is often clarifying but it does not immediately lead to a solution strategy, since the optimal solution of {\sc pomdp}s is highly computationally intractable even in discrete state and action spaces.

Practical approximation strategies for {\sc pomdp}s almost all rely on some form of receding-horizon control~\cite{bircher2016receding}.
The system makes a plan for a sub-problem under some assumptions, and begins executing it, gathering further information about the map as it goes.
When it receives an observation that invalidates its plan (e.g. an object in its path) or reaches its subgoal, the system makes a new plan based on the current robot configuration and map state.

When the objective is exploration, a typical strategy is some form of {\it frontier-based} or {\it next-best view planning} method~\cite{yamauchi1997frontier,dornhege2013frontier,bekris2007greedy}.
On each replanning iteration, a subgoal configuration is selected (a) that is reachable within known free space from the robot's current configuration and (b) from which some previously-unobserved parts of the workspace can be viewed.


When the objective is navigation, a typical replanning strategy is to be optimistic, planning a path to the goal that makes some assumptions about the true map and replanning if that assumption is invalidated.
Our contribution addresses the planning component of this approach.


\paragraph{Our work:}
We address a version of the problem in which we assume that the robot knows a map in advance.
This map is assumed to be accurate in the sense that the obstacles it contains are definitely present in the world.

We assume that the robot has some form of obstacle sensor, but make no assumptions about it except that, for any configuration of the robot, it can observe some (possibly disjoint) subset of the workspace and that this visibility function is known in advance.
We assume that observation and control are deterministic and the robot always knows its configuration.
The algorithms in this paper assume a robot that can reverse a path without sweeping through additional workspace: this is true of holonomic robots, but also round differential-drive robots, for example.

The planner we present in section~\ref{ssec:vampalg} would be used in a ``trust but verify'' replanning framework, in which we assume, optimistically, for the purposes of planning, that the obstacles in our current map are, in fact, the only obstacles.
This assumption makes it worthwhile to try to plan a complete path to the goal.
However, because we are not certain that these are the only obstacles and because we wish to guarantee the robot's safety, we will seek a {\it visibility-aware} path to the goal, in which the robot never moves into space that has not been verified to be free during some previous part of its path (we formalize this criterion more carefully in section~\ref{sec:definition}).
The robot could then execute this path until it observes an obstacle that invalidates the path.
At that point, it would insert that obstacle into its map and re-plan.
The focus of this paper is on methods for planning optimistic visibility-aware trajectories.

We provide several planning algorithms that take a sampled feasible-motion graph as input (e.g. a PRM, or state lattice).
They are guaranteed to be correct and complete on the given graph, and may be {\it resolution-complete} or {\it probabilistically complete} depending on the strategy used to augment the samples in the graph in case of failure.
We do not yet make a claim about the completeness of the entire receding-horizon control policy.

\section{Definitions}

\label{sec:definition}
We will focus on problems in which the robot is holonomic or reversible, so the state of the robot can be modeled only in terms of its configuration.
We will use $\powerset{X}$ to denote the {\it powerset} of set $X$ (the set of all possible subsets.)

\newcommand{\cfree}{C_\text{free}}
\newcommand{\wofree}{W_\text{ofree}}
\newcommand{\wobs}{W_\text{obs}}
\newcommand{\qgoal}{Q_\text{goal}}
\newcommand{\eo}{{E_\text{ofree}}}

A \vamp{} problem instance is a tuple $\left(W, C, V, \wobs, q_0, \qgoal, v_0 \right)$ where: $W = R^2$ or $R^3$ is the workspace; $C$ is the configuration space of the robot; $V : C \rightarrow \powerset{W}$ is a {\it visibility function}, mapping robot configurations into (possibly disconnected) subsets of workspace that are visible from that configuration, conditioned on the known obstacles;  $S: C \rightarrow \powerset{W}$ is a {\it swept volume function}, mapping robot configurations into subsets of workspace occupied by the robot in that configuration; $\wobs{} \in \powerset{W}$ is the subset of workspace that is known to contain obstacles; $q_0 \in C$ is the initial robot configuration; $\qgoal : C \rightarrow \text{Bool}$ is a function from configurations to Booleans indicating whether the configuration satisfies the goal criteria; and $v_0 \subseteq W$ is a region of workspace that has already been viewed and confirmed free (by default, it will be equal to $V(q_0)$, but for some robots it will need to be more, in order to allow any initial movement.)
We are assuming that motion is continuous-time (so all configurations along a path must be previously viewed) but that perception is discrete-time (so new views are only gained at the end of each primitive trajectory).

We will define some additional useful quantities, in terms of the basic elements of a \vamp{} problem instance.
We extend the definition of swept volume to a path segment, so $S\left(q_{i},q_{j}\right)\subseteq W$ is the region swept by moving from $q_{i}$ to $q_{j}$ via a primitive trajectory (such as straight line in $C$-space).
We further extend the definitions of $S$ and $V$ to paths: $S\left(\left[q_{1}\cdots,q_{n}\right]\right)=\bigcup_{i=1\cdots n-1}S\left(q_{i},q_{i+1}\right)$ and $V\left(\left[q_{1}\cdots,q_{n}\right]\right)=\bigcup_{i=1\cdots n}V\left(q_{i}\right)$.

An edge between configurations $q_i$ and $q_j$ is {\it optimistically traversible} if its swept volume does not intersect a workspace obstacle; the set of optimistic edges, then is $\eo{} = \{ (q_1, q_2) \mid S(q_1, q_2) \cap \wobs = \emptyset\}$.
A path $\left[q_{1},\cdots q_{n}\right]$ is {\it feasible} for the problem if and only if every edge is optimistically traversible and only moves through parts of the workspace that have been previously viewed: $(q_{i}, q_{i+1}) \in \eo{}$ and $S\left(q_{i}, q_{i+1}\right)\subseteq v_0 \cup V\left(\left[q_{1},\cdots,q_{i}\right]\right)$, for all $i \in \{0, \ldots, n-1\}$.
We refer to this second condition as the {\it visibility constraint}.



\section{Planning algorithms}

We present several algorithmic approaches to the \vamp{} problem, beginning with a computationally inefficient method that produces very high quality plans, and then exploring alternative strategies that are more computationally tractable.

In all these algorithms, we assume a given finite graph $(Q, E)$ where $Q \subset C$ is a set of configurations and all edges $(q_1, q_2) \in E$ are collision-free with respect to $\wobs$, so $S(q_1, q_2) \cap \wobs = \emptyset$.
This graph may be, for example, a fixed-resolution grid or a probabilistic road-map~\cite{Kavraki96}.
Any of our \vamp{} algorithms can be augmented by an outer loop that increases the resolution or sampling density of the graph.

\subsection{Belief-space search}

The most conceptually straightforward approach to this problem is to perform a tree-search in {\it belief space}.
Roy et al. \cite{prentice2009belief,bry2011rapidly} have pioneered techniques of this type in uncertain robot motion planning problems.
The basic idea is that a state of the whole system consists of a robot configuration and a current belief state, with the robot configurations drawn from the set $Q$.
In this problem, the belief state $v \in \powerset W$ is the region of the workspace that has been observed by the robot to be collision-free.

\begin{algorithm}
	\caption{\textproc{Vamp\_Bel}$((Q, E), V, q_0, \qgoal, v_0$)}
	\label{bsAlg}
	\begin{algorithmic}
		\State $s_0 \gets (q_0, v_0)$ \Comment initial state
		\State $g((q, v)) \gets \qgoal(q)$ \Comment goal test
		\State $A((q, v)) \gets \{q' \mid (q, q') \in E \;\;\text{and}\;\;S(q, q') \subset v \}$ \Comment legal actions in state $(q, v)$
		\State $T((q, v), q') \gets (q', v \cup V(q'))$ \Comment state transition function
		\State $H((q, v)) \gets \alpha \lvert S(\text{\sc mp}((Q, E), q, \qgoal)) \setminus v \rvert$ \Comment heuristic
		\State \Return $A^*(s_0, g, A, T, H)$
	\end{algorithmic}
\end{algorithm}

Procedure {\sc Vamp\_Bel} provides an implementation of the belief-space search via a call to $A^*$.
The procedure is given, as input, the graph $(Q, E)$, visibility function $V$, initial configuration $q_0$, goal test $\qgoal{}$, and initial visible workspace $v_0$.
The set of legal actions $A$ that can be taken from state $(q, v)$ is the set of outgoing edges from configuration $q$ that have the property that their swept volume is contained in the previously-viewed region of the workspace $v$.
The transition function $T$  moves along the edge to a new configuration and augments $v$ with the region of configuration space visible from the new configuration.
In order to drive the search toward a goal state, we define a heuristic which is based on a visibility-unaware path $\text{\sc mp}((Q, E), q, \qgoal)$ obtained by solving the underlying motion-planning problem to the goal.
The size of the swept volume of that path that has not yet been viewed is used as a measure of the difficulty of the remaining problem; $\alpha$ is a constant that makes the units match.
Note that, in this search, it is possible for an optimal path to visit the same configuration more than once (with different visibility states $v$).
Nonetheless, the search space is finite given finite $Q$, because only finitely many possible visibility states can be reached (at most one for each {\it set} of configurations in $Q$).

\begin{theorem}
  Algorithm {\sc Vamp\_Bel} is correct and complete with respect to
  configuration space graph $(Q, E)$.
\end{theorem}
\begin{proof}
  \ifForWAFR
  \proofInSupp
  \else
  It is correct, because if it returns a path, that path is
  a feasible path to a goal state.  The $A$
  function only allows the robot to move through space that has
  already been made visible along the path, so the steps are all
  feasible, and $A^*$ ensures that the final configuration satisfies
  the goal test.  It is complete, because the search space is finite,
  no feasible actions are ever disallowed, and $A^*$ is complete.
  \qed
  \fi
\end{proof}

This algorithm is computationally very complex even on modest graphs because the search must consider distinct paths that reach a given robot configuration.
The search can be pruned by using a {\it domination criterion}:  state $(q, v_1)$ dominates $(q, v_2)$ if $v_1 \subseteq v_2$, which means that if the search visits a state that is dominated by a state that has already been expanded, it can discard the dominated state.
In our experiments, this condition did not occur frequently enough to be useful; different paths will see slightly different regions.
On the other hand, in the setting of Bry and Roy~\cite{bry2011rapidly}, the belief space is lower dimensional (covariance matrices of the dynamical state space) and so domination happens much more frequently and makes the search tractable.

We implemented this algorithm with a very computationally cheap domination criterion that eliminates paths that revisit configurations without having visited any new configurations since the last visit (this eliminates looping paths, among others).
For {\sc HallwayEasy}, figure~\ref{narrowCorner}, the heuristic is very effective at guiding the search---a solution is found in under \SI{10}{\sec} after expanding 500 search nodes.
However, on {\sc HallwayHard}, figure~\ref{backup}, no solution was found after expanding 440K nodes, with 2 million nodes on the queue, with a computation time of over \SI{2}{\hour}.

\subsection{Local-visibility searches}

At the opposite end of the spectrum of approaches to the \vamp{} problem are methods that perform search directly in configuration space, as opposed to the problem's natural state space, which must include information about regions of observed space.
These local approaches are not complete in general, but may be complete for some robots;  they will prove useful as a subroutine in later algorithms.
{\sc Vamp\_Step\_Vis} is defined in algorithm~\ref{alg:step_vis}.
The basic version of the method has the same arguments as {\sc Vamp\_Bel}, but it may also be used in {\it relaxed} mode, which is signaled by parameter $\text{relaxed} = \text{true}$, and makes use of an additional argument $O \subset W$, which is a workspace region considered to be {\it out of bounds}.
In any mode, it may optionally be given a heuristic function.

\begin{algorithm}
	\caption{\textproc{Vamp\_Step\_Vis}$((Q, E), V, q_0, \qgoal, v_0, H = 0, \text{relaxed} = \text{False}, O = \emptyset)$}
	\label{alg:step_vis}
	\begin{algorithmic}
		\State $s_0 \gets q_0$ \Comment initial state
		\If{relaxed} {}  \Comment legal actions and cost
		\State $A(q) \gets \{q' \mid (q, q') \in E  \;\;\text{and}\;\;S(q, q') \cap O = \emptyset\} $
		\State $C(q, q') \gets  \lVert q - q' \rVert_2 \cdot
			\left(\text{{\bf if} $S(q, q') \subseteq (v_0
						\cup V(q))$ {\bf then} $1$ {\bf else}
					$\lvert S(q,q') \setminus (v_0 \cup V(q))\rvert$}\right)$
		\Else
		\State $A(q) \gets \{q' \mid (q, q') \in E \;\;\text{and}\;\;S(q, q') \subseteq (v_0 \cup V(q))\} $
		\State $C(q, q') \gets \lVert q - q' \rVert_2$
		\EndIf
		\State $T(q, q') \gets q'$ \Comment state transition function
		\State \Return $A^*(s_0, \qgoal, A, T, H, C)$
	\end{algorithmic}
\end{algorithm}

When it is not relaxed, it allows traversal of any edge $(q, q') \in E$ whose swept volume is entirely contained in the union of the initial visibility space $v_0$ and the region of workspace visible from $q$, $V(q)$.
For some combinations of robot kinematics and visibility, this algorithm will be complete.
For example, a robot with a wide field of view will always be able to see the space it is about to move into.
However, this method does not suffice for robots that can move into space that is not immediately visible to them.
Relaxed mode is used only to compute intermediate subgoals, but never for executable paths; in it, the robot is allowed to move into areas of the workspace that have not yet been seen, but these motions incur an extra cost.
It is not, however, allowed to collide with the out-of-bounds region under any circumstance, a feature used in the {\sc Tourist} algorithm of section~\ref{ssec:vampalg}.
Ideally, the relaxed planner would solve a {\it minimum-constraint removal} problem~\cite{hauser2014minimum}, keeping track of regions that have ever been violated, and not double-counting the regions that experience repeat violations.
This is a very computationally difficult sub-problem, so we simply penalize total distance traversed through unviewed regions.

An very useful extension of {\sc Vamp\_Step\_Vis} is the algorithm {\sc Vamp\_Path\_Vis}.
It is, fundamentally, a search in the space of configurations, as in {\sc Vamp\_Step\_Vis}, and so for example the tests for whether a node has been visited inside $A^*$ are made on the configuration alone.
However, as in {\sc Vamp\_Bel}, we ``decorate'' each node that is added to the tree with the visibility $v$ computed along the path to that node.
However, whatever visibility we have the first time we expand a configuration $q$ is the visibility that will be permanently associated with it.
This method can solve more problem instances than {\sc Vamp\_Step\_Vis}, and is always correct, but it is incomplete (because it might commit to a path to some configuration that is not the best in terms of visibility, and it cannot contemplate paths that must revisit the same configuration).
We will use it extensively as a subroutine in our final algorithm in section~\ref{ssec:vampalg}

\begin{figure} \begin{center} \subfloat[] { \includegraphics[width=0.20\textwidth]{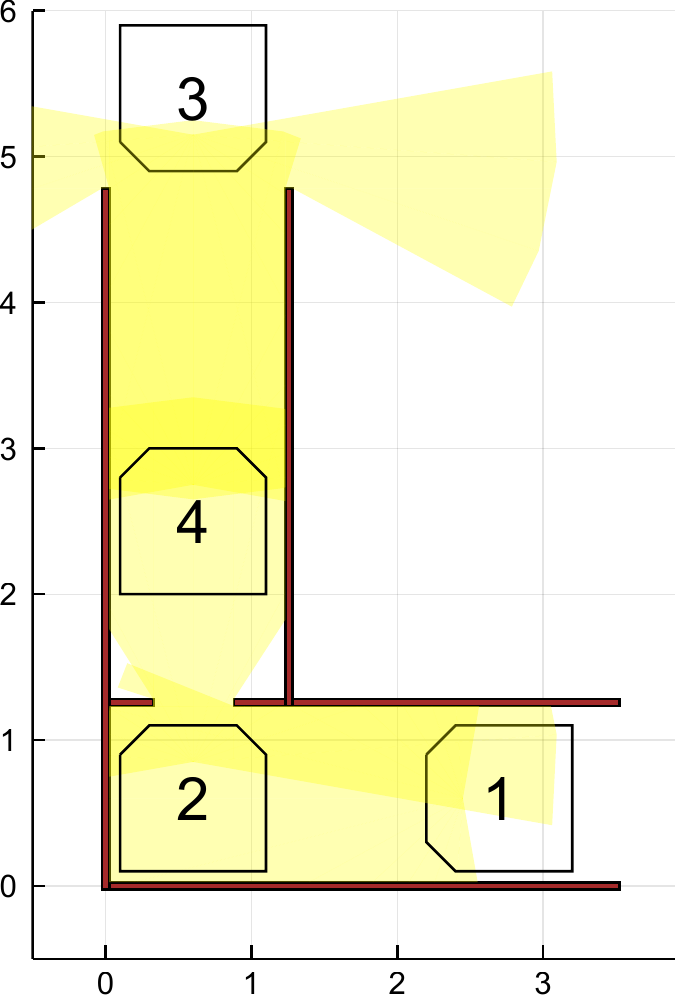} \label{subfig:two-hallway-subgoals} } \subfloat[] { \includegraphics[width=0.25\textwidth]{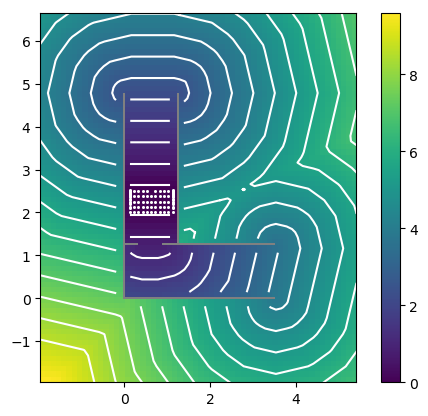} \label{subfig:achieve-vis-heuristic} } \subfloat[] { \includegraphics[width=0.25\textwidth]{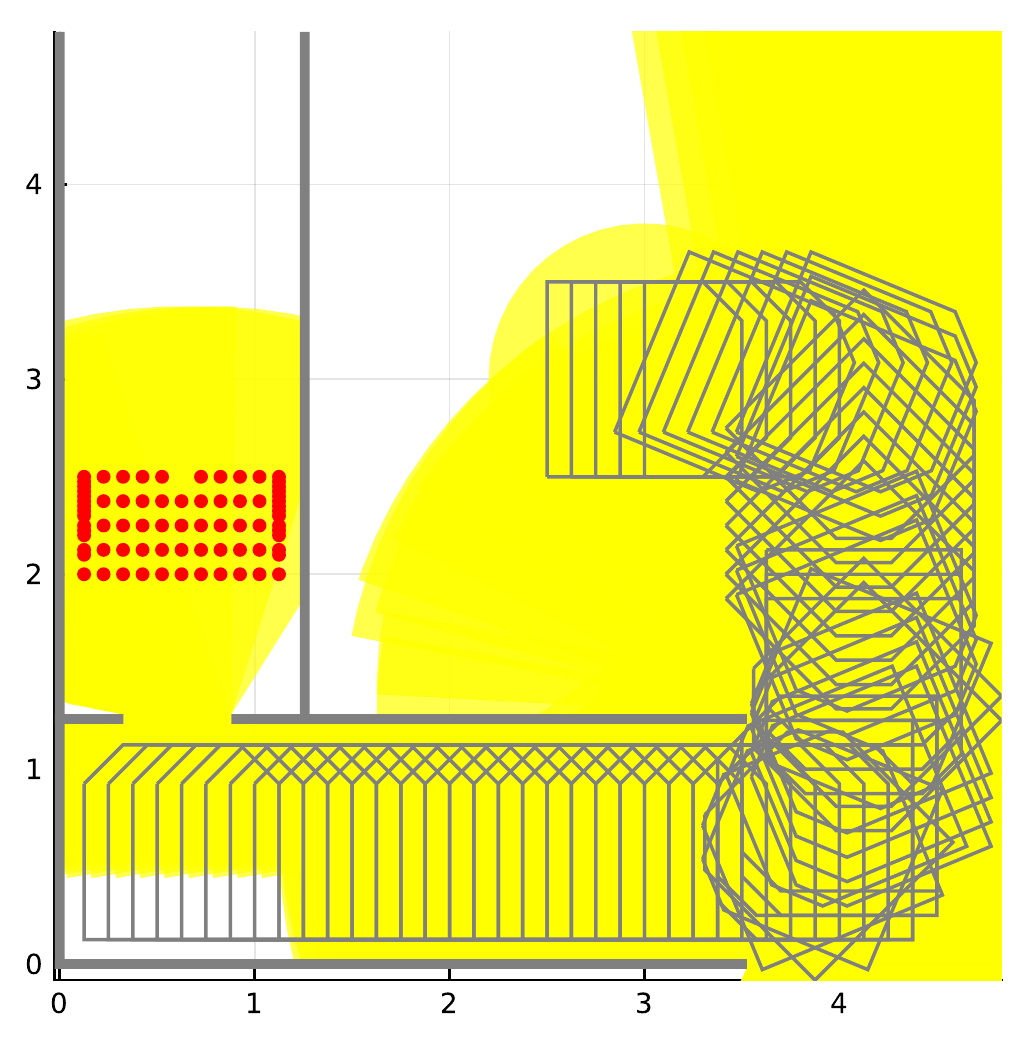} \label{subfig:goal_regression1} } \subfloat[] { \includegraphics[width=0.25\textwidth]{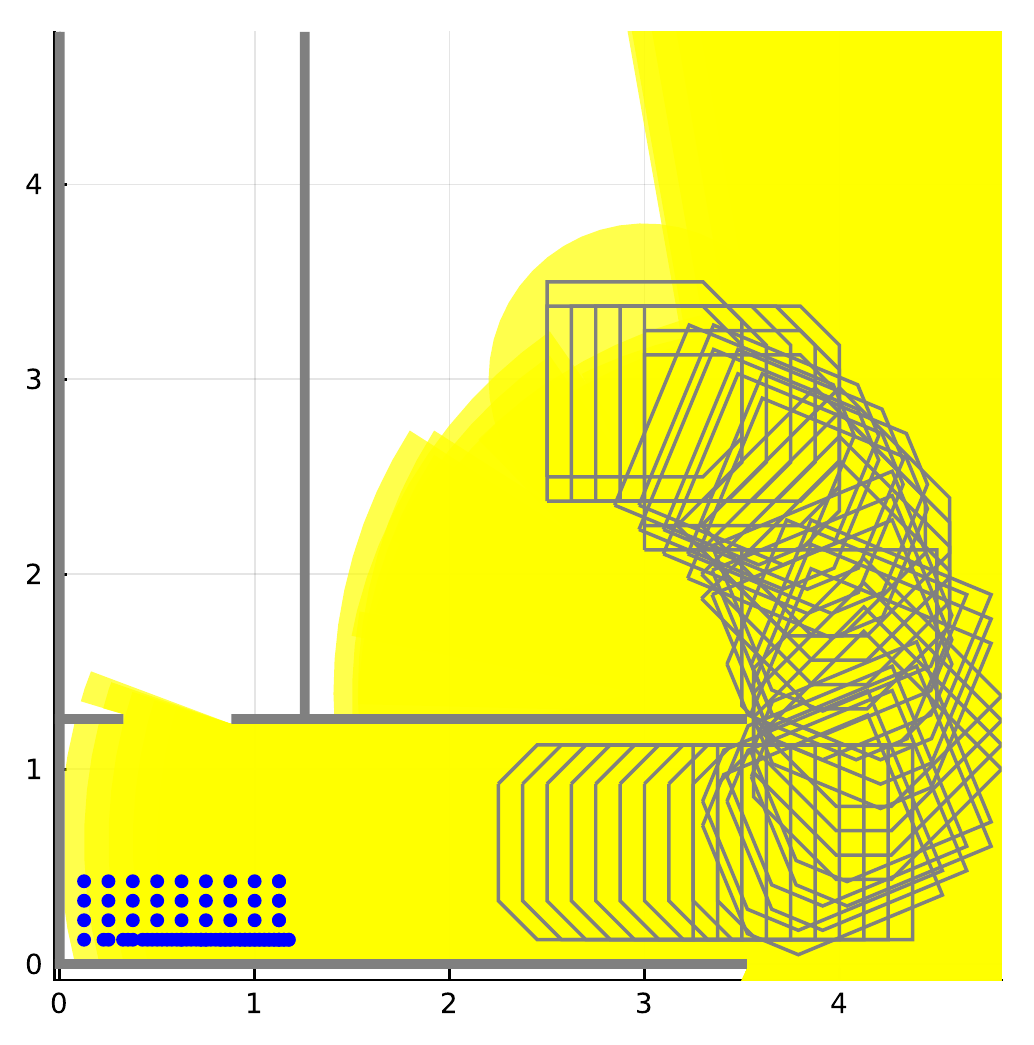} \label{subfig:goal_regression2} } \end{center}

	\caption{(a) The {\sc TwoHallway} domain with some hand-generated subgoals.
		(b) Level-sets and heat-map of field $F$ used to compute heuristic for {\sc Tourist} algorithm. (c, d) Two levels of {\sc Vavp}. }
	\label{fig:two-hallway-subgoals}
	\label{fig:regression}
\end{figure}

\subsection{Tree-visibility tour}

One observation is that, under the \vamp{} assumptions, visibility is monotonic:  that is, as the robot moves through the world, after any discrete path step and observation, $v_{t-1} \subseteq v_t$.
In addition, it is order-independent: $V([q_0, \ldots, q_n]) = V(\text{perm}[q_0, \ldots, q_n])$ where perm is any permutation of the sequence of configurations.
These observations lead us to an algorithm that is complete and much more efficient at finding solution paths for \vamp{} problems than belief-space search, although we will find that it will generally be unsuitable in practice.

Rather than associating a new visibility region $v$ with each state in the search, we will maintain a single, global $v \in \powerset{W}$ and carry out a search directly in $Q$.
The search can only traverse an edge if its swept volume is contained {\it in the workspace that has been viewed during the entire search process up until that time}.
Once this process reaches a goal state, the tree, in the order it was constructed, is used to construct a solution path.
Pseudo-code is shown in~\ref{alg:treevis}.

\begin{algorithm}
	\caption{\textproc{Vamp\_Tree} $((Q, E), V, q_0, \qgoal, v_0)$}
	\label{alg:treevis}
	\begin{algorithmic}
		\State $\text{agenda} \gets [\text{$(q_0, q')$ for $(q_0, q') \in E$}]$
		\State $\text{visited} \gets [q_0] ; \;\;
			T \gets [\;] ; \;\;
			v \gets v_0$
		\While {agenda is not empty}
		\State $(q_s, q_e) \gets \text{pop}(\text{agenda})$
		\If {$q_e \in \text{visited}$} {\bf continue}
		\EndIf
		\If {$S(q_s, q_e) \subseteq v$}
		\State $\text{visited}.\text{append}(q_e)$ \Comment add conf to path
		\State $T.\text{append}((q_s, q_e))$ \Comment add edge to tree
		\If {$\qgoal(q_e)$} {\bf break}            \EndIf
		\State $v \gets v \cup V(q_e)$ \Comment add new visibility
		\State $\text{agenda}.\text{extend}([\text{$(q_e, q')$ for $(q_e, q') \in E$}])$ \Comment add outgoing edges to agenda
		\Else
		\State $\text{agenda}.\text{append}((q_s, q_e))$ \Comment save edge for reconsideration
		\EndIf
		\EndWhile
		\If {{\bf not} $\qgoal(q_e)$} \Return Failed
		\EndIf

		\State $p \gets [q_0]; \;\;
			q_\text{curr} \gets q_0$
		\For {$i \in [1..\text{len}(\text{visited})]$}
		\State $q_\text{next} \gets \text{visited}[i]$ \Comment link configurations using previously-enabled edges
		\State $p.\text{extend}(\text{shortest\_undirected\_path}(q_\text{curr}, q_\text{next}, T[0:i])[1:])$
		\State $q_\text{curr} \gets q_\text{next}$
		\EndFor
		\State \Return p
	\end{algorithmic}
\end{algorithm}

It proceeds in two phases.
First, it constructs a search tree,  where the extension from a point in the tree is made only within the region that has been visible from any configuration previously visited in the search.
Second, it constructs a path that visits all of the configurations, in the order in which they were added to the tree, and returns that path.
The tree search is slightly unusual, because which edges in the graph can be traversed depends globally on all search nodes in the tree.
For this reason, we perform a queue-based search, keeping an agenda of {\it edges}, rather than nodes.
If an edge is selected for expansion, but is not yet traversible, it is added back to the end of the agenda for reconsideration after some more of the tree has been grown.
When a goal state has been reached, we extract a path from the tree.
This path will visit the configurations in the same order that they were visited by the search, but they must be connected together via paths in the tree that existed at the point in the search when the configuration was visited.


\begin{theorem} {\sc Vamp\_Tree} is correct and complete with respect to
  the configuration-space graph $(Q, E)$ for any robot such that
  $S(q_1, q_2) = S(q_2, q_1)$ for all $q_1, q_2$.
\end{theorem}
\begin{proof}
  \ifForWAFR
  \proofInSupp
  \else
  It is correct because, if it returns a path, it is a feasible path
  to a goal state.  The set of edges $(q_1, q_2)$ added to $p$ on
  iteration $i$ of the path-construction phase have the property that
  either $(q_1, q_2)$ or $(q_2, q_1)$ is in $T[0:i-1]$, which, by
  construction of the tree $T$ and the reversibility assumption in the
  theorem statement, means that
  $S(q_1, q_2) \subseteq V(\text{visited}[0..i-1])$.  This, in turn,
  implies that the path is feasible.  The last configuration in
  $\text{visited}$ clearly satisfies the goal test, and it is also the
  last configuration in the returned path $p$.

  To show that it is complete, we must show that if a feasible path to
  a goal state exists in $(Q, E)$, the search will find it (or another
  feasible path).  Assume $[q_0, \ldots, q_n]$ where
  $\qgoal(q_n) = \text{\it true}$ is a feasible path and assume, for
  the sake of contradiction, that the first {\bf while} loop cannot
  add all of the configurations $[q_0, \ldots, q_n]$ to {\it visited}.
  Then there must be a point in that loop when $[q_0, \ldots, q_i]$
  are in {\it visited} for some $0 \leq i < n$ but the algorithm
  cannot reach $q_{i+1}$.  We know by the assumption that this is a
  feasible path, so $(q_i, q_{i+1}) \in E$ and
  $S(q_i, q_{i+1}) \subseteq V([q_0, \ldots, q_{i}]) \cup v_0$, which means
  $q_{i+1}$ must be in $A(q_i)$.  We also know that $(q_i, q_{i+1})$
  will be in {\it agenda}, because $q_i$ is in {\it visited} so it was
  added, but $q_{i+1}$ is not in visited, so that edge has not been
  removed from the agenda.  But if $(q_i, q_{i+1})$ is in the agenda
  and $q_{i+1}$ is in $A(q_i)$, then $q_{i+1}$ can be added to {\it
    visited}, and so we reach a contradiction.  Thus, we have shown
  that after the {\bf while} loop, $q_n$ has been reached, and so the
  algorithm continues to the second phase.  The only possible failure
  mode of the second phase is if {\it shortest\_undirected\_path} fails;
  but, by construction, both $q_\text{curr} = \text{visited}[i-1]$ and
  $q_\text{next} = \text{visited}[i]$  are in $T[0:i]$, as are paths
  from $q_0$ to each of them.  Thus we know that there is, at worst, a
  path going from $q_\text{curr}$ up to $q_0$ and back down to
  $q_\text{next}$, by the reversibility assumption.  So this loop will
  terminate and a path $p$ will be returned.
  \qed
  \fi
\end{proof}

\subsection{Visibility preimage backchaining} \label{ssec:vampalg}

Our final approach to this problem is to perform a much more goal-driven search to observe parts of the workspace that will make desired paths feasible.
This algorithm is motivated by the observation that goals can be decomposed into subgoals.
Figure~\ref{subfig:two-hallway-subgoals} shows a goal configuration marked $4$.
To make this configuration visible, the robot must stand at $2$ and $3$.
Finally, to make $2$ visible, the robot must stand at $1$.
It is interesting to note that {\sc Vamp\_Path\_Vis} can plan efficiently to visit these subgoals in order, if they are provided.
With this motivation in mind, we describe a general algorithm that has several special cases of interest, described in section~\ref{sec:experiments}.

We make use of the {\sc Tourist} algorithm, whose goal is to see some part of a given previously-unobserved region of workspace.
It uses a local-visibility algorithm to find a path, but where the goal test for a configuration is that it is possible to observe some previously unobserved part of the workspace from there.
A critical aspect to making this search effective is to use a heuristic that drives progress toward the objective of observing part of a region of interest, $R$.
We begin by computing a scalar field, $F$, in workspace, of the shortest distance from location $x$ to a point in $R$.
Then, the heuristic is $H(q) = \min_{x \in V(q)} F(x)$, which assigns 0 heuristic value to a configuration that can see part of $R$ (because it will be able to see a workspace point $x$ with $F(x) = 0$) and increasingly higher heuristic values to configurations that can only see points that are "far" in the sense of $F$ from $R$.
Computing $F$ is relatively inexpensive, and it effectively models the fact that visibility does not go through walls.
This heuristic is illustrated in figure~\ref{subfig:achieve-vis-heuristic}: the black nodes are the workspace target region $R$.
The figure illustrates the level sets of $F$.

\begin{algorithm}
	\caption{\textproc{Tourist}$((Q, E), V, q_0, R, v_0, \text{relaxed} = \text{false}, O = \emptyset$)}
	\label{tourist}
	\begin{algorithmic}
		\State $H(q) = \min_{x \in V(q)}
			F(x)$ \Comment where $F$ is distance field \State \Return \Call{Vamp\_Path\_Vis}{$(Q, E), V, q_0, \lambda q.
				(V(q)\cap R) \not = \emptyset, v_0, H, \text{relaxed}, O$}
	\end{algorithmic}
\end{algorithm}

Now we can describe the {\sc Vamp\_Backchain} algorithm, with pseudo-code shown in algorithm~\ref{alg:backchain}.
The main loop of the algorithm is in lines~\ref{lst:line:vamp-start}--\ref{lst:line:vamp-end}.
It keeps track of $p$, the solution path it is constructing, $v$, the region of workspace that has been viewed by $p$, and $q$, the configuration at the end of $p$.
On every iteration, it checks to see whether a goal state is reachable from $q$ with the current visibility $v$.
If so, it appends the path that does so to $p$ and returns a final solution.
If that test fails, then it generates a path that is guaranteed to increase $v$ (if the problem is feasible), ideally in a way that makes it easier to reach a goal configuration.
In line~\ref{lst:line:relaxed-move}, we find a {\it relaxed plan} $p_\text{relaxed}$ that reaches a goal state, preferring to stay inside $v$, but allowing excursions if necessary.
Now, our objective is to find a path $p_\text{vis}$ that will observe some previously-unobserved part of the swept volume of $p_\text{relaxed}$, by calling procedure {\sc Vavp}.
If that call fails, then we fall back on an undirected exploration strategy, to view any part of the unviewed workspace.
Once we have found a view path, we update $p$, $v$ and $q$ based on $p_\text{vis}$, test to see if we can now find a path to the goal, etc.

\begin{algorithm}
	\caption{\textproc{Vamp\_Backchain}$((Q, E), W, V, q_0, \qgoal, v_0)$}
	\label{alg:backchain}
	\begin{algorithmic}[1]
			\Procedure{Vavp}{$q, R, v, O  = \emptyset$}
					\State $p_\text{vis} \gets \Call{Tourist}{q, R, v}$ \label{lst:vavp:line:unrelaxed-tourist}
					\If{$p_\text{vis} \not = Failed$}
							 {\bf return} $p_\text{vis}$
					\EndIf
					\State $O_\text{new} = O \cup R$
					\State $p_\text{relaxed} \gets \Call{Tourist}{q, R, v, \text{relaxed} = \text{true}, O = O_\text{new}}$ \label{lst:vavp:line:relaxed-tourist}
					\If{$p_\text{relaxed} \not = Failed$}
							\State $p_\text{vis} \gets \Call{Vavp}{q, S(p_\text{relaxed}) \setminus v, v, O = O_\text{new}}$
							\If{$p_\text{vis} \not = Failed$}
									 {\bf return} $p_\text{vis}$
							\EndIf
					\EndIf
					\State {\bf return} \text{Failed}
			\EndProcedure
			\State

			\State $p \gets [\;]; \;\;
							v \gets v_0; \;\;
							q \gets q_0$ \label{lst:line:vamp-start}
			\While{True}
					\State $p_\text{final} \gets \Call{Vamp\_Path\_Vis}{q, \qgoal, v}$ \label{lst:bc:line:unrelaxed-move}
					\If{$p_\text{final} \not = \text{Failed}$}
							 {\bf return} $p + p_\text{final}$
					\EndIf
					\State $p_\text{relaxed} \gets \Call{Vamp\_Path\_Vis}{q, \qgoal, v, \text{relaxed} = \text{true}}$ \label{lst:line:relaxed-move}
					\State $p_\text{vis} \gets \Call{Vavp}{q, S(p_\text{relaxed}) \setminus v, v}$
					\If {$p_\text{vis} = \text{Failed}$} {$p_\text{vis} \gets \Call{Tourist}{q, W \setminus v, v}$} \label{lst:line:see-anything}
					\EndIf
					\If {$p_\text{vis} = \text{Failed}$} {{\bf return} Failed} \label{lst:line:vamp-end}
					\EndIf
					\State $p \gets p + p_\text{vis}; \;\;
									v \gets v \cup V(p_\text{vis}); \;\;
									q \gets p_\text{vis}[-1]$
			\EndWhile
	\end{algorithmic}
	\end{algorithm}

The {\sc vavp} sub-procedure takes a configuration $q$, region of interest $R$, previously viewed volume $v$ and out-of-bounds area $O$ and returns a path that will view some part of $R$ without colliding with $O$, or, in a recursive call, view an unviewed part of the swept volume of a path that will view some part of $R$ without colliding with either $O$ or the new target region, etc.   If it cannot find any such recursively useful view path, it fails.
For visual simplicity we are omitting arguments $(Q, E), V)$ from calls to {\sc Tourist} and {\sc Vamp\_Path\_Vis}.

\label{sec:backchain}

Figure~\ref{fig:regression} illustrates the operation of {\sc Vamp\_Backchain} in the {\sc TwoHallway} domain.
It is given a goal to see the region at the end of the vertical hallway (red points in figure~\ref{subfig:goal_regression1}).
The hallway is keyed, and the robot can only see the region through the peephole.
	{\sc Vavp} generates a relaxed plan (gray) to see these
points.
This relaxed path is in violation, and requires regions be made visible (blue points in figure~\ref{subfig:goal_regression2}) before it can be executed.
	{\sc
		Vavp} recurses one level, with the blue region as the goal, and both
marked "out of bounds", and generates the gray path in figure~\ref{subfig:goal_regression2}.
This path satisfies the full constraints, and it is returned by {\sc Vavp}.
Note that the returned path does not satisfy the original goal, but achieves visibility to enable solving the original goal in a later call to {\sc Vavp}.

Two difficult examples that motivate the structure of the {\sc Vamp\_Backchain} algorithm are illustrated in figure~\ref{fig:nesting}.

In figure~\ref{subfig:backward-reasoning-fail}, the robot must move to the dashed outline on the right.
It cannot do so with step-wise visibility (line~\ref{lst:bc:line:unrelaxed-move}), so it makes a relaxed plan (line~\ref{lst:line:relaxed-move}) to slide horizontally to the goal.
However, none of the swept volume of that relaxed plan can be viewed (line~\ref{lst:vavp:line:unrelaxed-tourist}) under normal visibility constraints, nor can we even generate a relaxed plan to view it (line~\ref{lst:vavp:line:relaxed-tourist}), so {\sc Vavp} fails.
We fall back on simply generating a path that views some part of the un-viewed workspace (line~\ref{lst:line:see-anything}) which yields the path shown by the unfilled robot outlines.
The ultimate solution to the problem is indicated by the robot outlines.

In figure~\ref{subfig:arbitrary-nesting}, we see an example illustrating the potential need for arbitrary recursive nesting.
In this case, the inner walls are transparent (so the robot can see through them, but it cannot move through them.)
The solution requires moving forward into the bottom-most hallway to clear it, then moving into it again sideways to look through the windows, thus clearing the hallway above it, and so on.

\begin{figure}
	\begin{center}
		\subfloat[] 
		{
			\includegraphics[width=0.45\textwidth]{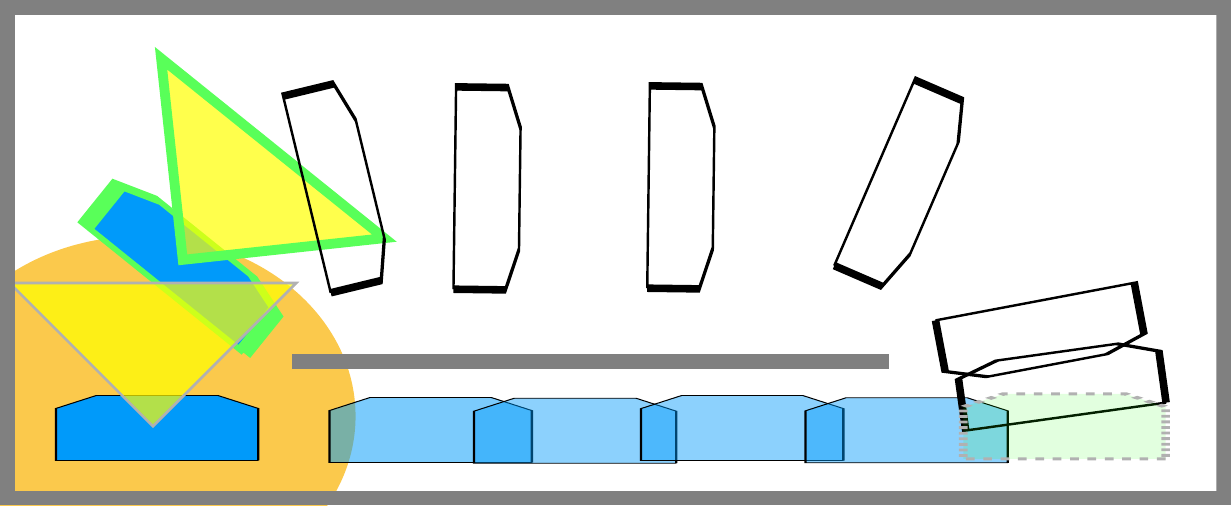}
			\label{subfig:backward-reasoning-fail}
		}
		\hfill
		\subfloat[] 
		{
			\includegraphics[width=0.29\textwidth]{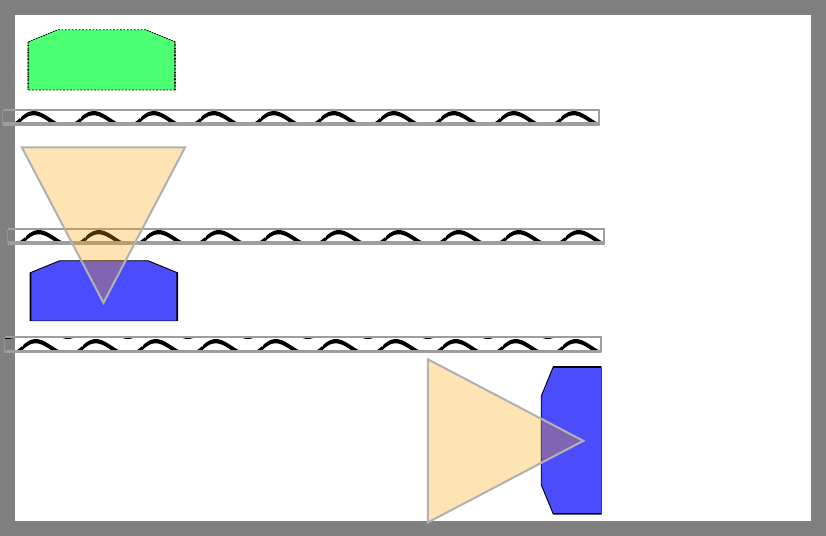}
			\label{subfig:arbitrary-nesting}
		}
	\end{center}
	\caption{Difficult examples for the {\sc Vamp\_Backchain} algorithm.}
	\label{fig:nesting}
\end{figure}

\begin{theorem}
  The algorithm {\sc Vamp\_Backchain} is correct and complete with respect to
  the configuration-space graph $(Q, E)$ for any robot such that
  $S(q_1, q_2) = S(q_2, q_1)$ for all $q_1, q_2$.
\end{theorem}
\begin{proof}
  \ifForWAFR
  \proofInSupp
  \else
  If the algorithm returns an answer, it is a feasible path to a goal
  state.  The path is feasible because it is a concatenation of paths
  made by non-relaxed calls to {\sc Vamp\_Path\_Vis} (either directly or
  via calls to {\sc Tourist}), and those paths are feasible by
  construction. The final call to {\sc Vamp\_Path\_Vis} guarantees that
  the final configuration satisfies $\qgoal$.

  To show that it is complete, we begin with some lemmas.

  Lemma 1.  The {\sc vavp} procedure is guaranteed to terminate, and
  either return a path that will visit a configuration that has not been
  reached before or fail.

  Lemma 2.  If there is a feasible path and the call in line~\ref{lst:bc:line:unrelaxed-move} and line~\ref{lst:line:relaxed-move} fails,
  and the call to {\sc vavp} fails, then the call to {\sc Tourist} in
  line~\ref{lst:line:see-anything} is guaranteed to return a path that will visit a configuration
  that has not been visited before.

  By Lemmas 1 and 2, on every iteration of the main {\bf while loop},
  either the call to {\sc Vamp\_Path\_Vis} succeeds and finds and returns a solution
  path, or, a sequence of configurations will be added to the path that
  causes some not-yet-viewed space to be seen.

  The while loop terminates.  If a path does not exist, then eventually
  all the space that can be seen will have been seen and the call on
  line~\ref{lst:line:see-anything} will fail and the algorithm will terminate with failure.  If a
  solution path $[q_0, \ldots, q_n]$ does exist, then because repeated
  calls on line~\ref{lst:line:see-anything} will eventually visit all configurations that can be
  reached on feasible paths, and therefore see all the space that can be
  seen, then there is a point at which all of $S([q_0, \ldots, q_n])$
  will be in $v$ and so a call to {\sc Vamp\_Path\_Vis} on line~\ref{lst:bc:line:unrelaxed-move} will return with a solution.
  \qed

  \fi
\end{proof}

\section{Experiments}
\label{sec:experiments}
In our experiments, we consider a planar robot (\SI{1}{\meter} $\times$ \SI{1}{\meter}) operating in a 2D workspace.
For all of the experiments, we discretize the robot motions and search on a 6-connected lattice ($\Delta x =\Delta y=\SI{0.125}{\meter}$, $\Delta \theta = \frac{2\pi}{16}$).
The depth of view of the visible region is \SI{2.5}{\meter}.
All swept volume and containment computations were performed by sampling points along the boundary of the robot.


We ran two versions of {\sc Vamp\_Backchain}.
VB$\infty$ corresponds to the algorithm as presented in \ref{ssec:vampalg}, with the difference that all sub-calls to the planners are relaxed versions.
We never call the un-relaxed planner, however we still verify the feasibility of paths before incorporating them into the final path.
This choice has the benefit of accelerating the search procedure, since we do not have to wait for {\sc Tourist} to return failure in situations where the search is incomplete.
In practice, the relaxed planners often return a feasible path if one exists, but occasionally they produce a violating path, which means subsequent searches may do unnecessary work to provide visibility in the violated region.

VB$1$ corresponds to {\sc Vamp\_Backchain}, but with a recursion depth limit.
In this variation, in the first call to {\sc Vavp}, lines 4-5 are skipped.
In recursive calls, lines 4 is executed and its result is returned.
Furthermore, instead of waiting for this call to {\sc Tourist} to fail we set a timeout, to trigger the subsequent call to {\sc Tourist}.

We run experiments on many instances of \vamp{} problems.
Instances vary in obstacle, start and goal states, and field of view of the vision sensor.

There are three combinations of obstacle layout and start/goal states, each exhibiting increasing problem difficulty: {\sc HallwayEasy} depicted in figure~\ref{narrowCorner}, {\sc HallwayHard} depicted in figure~\ref{backup}, and {\sc TwoHallway} depicted in figure~\ref{subfig:two-hallway-subgoals}, which contains a ``keyed'' vertical hallway, which can only be entered backwards.


For each experiment, we report search time, path length, and total number of nodes expanded in any subroutine searches.

\begin{table}[h]
	\begin{tabular}{llrrrrrr}
\toprule 
 &  & \multicolumn{2}{c}{fov=\ang{50}} & \multicolumn{2}{c}{fov=\ang{200}} & \multicolumn{2}{c}{fov=\ang{350}} \\
\cmidrule(lr){3-4} \cmidrule(lr){5-6} \cmidrule(lr){7-8}
 &  & VB1 & VB$\infty$ & VB1 & VB$\infty$ & VB1 & VB$\infty$ \\
Search time (s) & {\sc HallwayEasy} & 5.1 & 16.4 & 1.2 & 1.2 & 1.0 & 2.2 \\
 & {\sc HallwayHard} & 23.4 & 315.3 & 9.4 & 13.9 & 2.4 & 3.5 \\
 & {\sc TwoHallway} & 2868.3 & 281.3 & 638.5 & 220.9 & 1952.1 & 134.8 \\
\midrule 
Path length (m) & {\sc HallwayEasy} & 12.3 & 13.3 & 8.4 & 8.4 & 8.4 & 8.4 \\
 & {\sc HallwayHard} & 14.3 & 16.9 & 12.5 & 12.5 & 11.4 & 11.4 \\
 & {\sc TwoHallway} & 63.9 & 43.4 & 47.6 & 43.2 & 40.6 & 34.3 \\
\midrule 
Closed nodes & {\sc HallwayEasy} & 2578 & 9241 & 377 & 377 & 137 & 137 \\
 & {\sc HallwayHard} & 7667 & 40428 & 3436 & 4469 & 604 & 604 \\
 & {\sc TwoHallway} & 139484 & 64145 & 76083 & 62586 & 92184 & 44188 \\
\midrule 
\bottomrule 
\end{tabular}

\end{table}

{\sc TwoHallway} is designed to demonstrate the recursion capabilities
of \\{\sc Vamp\_Backchain}, so VB$\infty$ noticeably outperforms VB$1$ on it.
Because VB$1$ does not perform backchaining more than once, it relies on line~\ref{lst:line:see-anything} of {\sc Vamp\_Backchain}.
In practice, for problems exhibiting the nested dependency as in {\sc TwoHallway}, VB$1$ generates paths that view the whole space because the search cannot be guided through nested dependencies.

Situations in which VB$\infty$ performs worse are due to sub-optimal relaxed paths, which incur violations that could be avoided.

We also collected search times and tree size for the {\sc TreeVis} algorithm.
For {\sc TwoHallway}, it expands ~62,000 nodes and searches for 60 seconds.
Note that this does not include any time for generating a path.
The na\"{i}ve path  would include every edge in the tree, visiting every node in search order, which would never be a practical path.
Note additionally that {\sc TreeVis} is not a directed search, and so in domains where the workspace is large, it is unlikely that {\sc TreeVis} will be practical.


\section{Discussion}

\vamp{} instances are challenging when the domain requires plans that achieve visibility in order to perform a motion in the future.
We present two small instances, {\sc HallwayHard} and {\sc TwoHallway} that have this property.
Solving the problem directly in the belief space is computationally intractable.
We, instead, direct the search by relying on calls to constraint-relaxed plans.
The setting we consider in this paper is fully deterministic, and in future work we are interested in handing uncertainty on the pose of ``known'' objects, and uncertainty on the pose of the robot due to odometry and localization errors.

\ifForWAFR
\else
\section{Post-processing to minimize views}

Each of the algorithms returns a path of consecutive configurations
such that, if the robot were to take and process an image at every
configuration, the path would be safe.  However, when imaging requires
the robot to be stationary or the processing is slow, it is desirable
to minimize the number of images required while still guaranteeing
safety.  To select which configurations actually require an image to
be acquired and processed, we simply run a greedy set-cover algorithm,
and then annotate the configurations in the path to indicate whether
the robot should take an image there.

Many mobile-manipulation robots have heads that pan and tilt.  If the
head is such that moving it substantially changes the robot configuration
from a collision-avoidance perspective (e.g., it can periscope up and down) then
it may be necessary to include the degrees of freedom of the head in
the robot's configuration, $q$, and apply the algorithms in this paper
directly.  However, when moving the head makes a relatively small
change in the swept volume of the robot, planning for the head can be
decoupled from planning for the rest of the robot.  We do this in two
phases.  First, we run one of the \vamp{} algorithms of this paper in
the configuration space of the robot, but without including the head's
degrees of freedom.  We use a visibility function $V$ that includes
{\it the union all possible views} that can be obtained by moving the
head, given the rest of the configuration $q$.  When a path is
returned, we post-process by selecting not just what robot
configurations require an image but also which orientation(s) the head
should have when taking the images.  We accomplish this by
partitioning the viewable region of space into a finite set and
associate a head configuration with each element.  Then, when we do
the greedy set-cover algorithm, we run it over the product of the body
configurations in the path and the possible head configurations.
\fi

\makeatletter
\renewcommand\section{\@startsection{section}{3}{\z@}%
	{-3.25ex\@plus -1ex \@minus -.2ex}%
	{-1.5ex \@plus .2ex}%
	{\normalfont\normalsize\bfseries}}
\makeatother

\section*{Acknowledgements}
\begin{small}
We gratefully acknowledge support from NSF grants 1523767 and 1723381; from AFOSR grant FA9550-17-1-0165; from ONR grant N00014-18-1-2847; from Honda Research; and from Draper Laboratory.  Any opinions, findings, and conclusions or recommendations expressed in this material are those of the authors and do not necessarily reflect the views of our sponsors.
\end{small}

\ifusePlainBib
\bibliography{references}
\bibliographystyle{unsrt}
\else
\printbibliography
\fi
\end{document}